\documentclass[11pt]{article}

\usepackage[margin=1in]{geometry}
\usepackage{times}

\usepackage{amsmath,amsfonts,bm}

\def\eqref#1{equation~\ref{#1}}

\def\1{\bm{1}}

\def\vtheta{{\bm{\theta}}}

\def\vg{{\bm{g}}}

\def\vm{{\bm{m}}}

\def\vr{{\bm{r}}}

\def\vx{{\bm{x}}}
\def\vy{{\bm{y}}}

\def\mA{{\bm{A}}}

\def\mJ{{\bm{J}}}

\DeclareMathAlphabet{\mathsfit}{\encodingdefault}{\sfdefault}{m}{sl}
\SetMathAlphabet{\mathsfit}{bold}{\encodingdefault}{\sfdefault}{bx}{n}

\usepackage{amsmath}
\usepackage{amssymb}
\usepackage{amsthm}
\usepackage{booktabs}
\usepackage{graphicx}
\usepackage{multirow}
\usepackage{xcolor}
\usepackage[authoryear,round]{natbib}
\setcitestyle{authoryear,round,citesep={;},aysep={,},yysep={;}}
\usepackage{hyperref}
\usepackage{url}
\usepackage{longtable}
\usepackage{array}
\usepackage{algorithm}
\usepackage{algorithmic}

\newtheorem{theorem}{Theorem}[section]

\title{Loss-Guided Pretraining Data Selection for Time-Series Foundation Models}

\author{
Yike Li \qquad Shaoxu Song\thanks{Corresponding author. Homepage: \url{https://sxsong.github.io/}} \qquad Jianmin Wang\\
Tsinghua University\\
\texttt{liyike25@mails.tsinghua.edu.cn}\\
\texttt{sxsong@tsinghua.edu.cn} \quad \texttt{jimwang@tsinghua.edu.cn}
}
\date{}

\begin{document}

\maketitle

\begin{abstract}
Time series foundation models (TSFMs) are pretrained on heterogeneous collections containing billions of observations, yet their training windows are typically sampled without estimating whether they provide useful learning signal. 
We introduce a static data-selection framework that scores each window with a reference forecaster and retains an intermediate interval within every source dataset. Specifically, we connect forecasting loss to optimization difficulty by showing
that normalized squared loss controls the per-sample gradient norm under a local Jacobian condition. 
We then define a reference loss score and apply dataset-stratified selection to preserve the diversity of samples.
Across various TSFM architectures, 
our method outperforms random selection by an absolute margin and even 
improves both relative MASE and CRPS over full-data pretraining by retaining fewer candidate pretraining windows.
Further analyses show strong cross-scale and cross-architecture score correlations, indicating that a small reference model can often select data for larger targets, provided that the reference and target share compatible difficulty orderings.

\end{abstract}

\section{Introduction}

Time-series foundation models (TSFMs)  aim to transfer temporal structure learned from large, heterogeneous corpora to unseen forecasting tasks. Recent systems such as TimesFM, Chronos, and Moirai demonstrate that a single pretrained model can forecast across domains, sampling frequencies, context lengths, and prediction horizons without task-specific fitting \citep{das2024timesfm,ansari2024chronos,woo2024moirai}. This universal forecasting paradigm shifts a major part of model development from downstream architecture design to corpus construction. As pretraining collections grow, however, the common assumption that every available window should receive equal attention becomes increasingly costly and scientifically questionable.

Time series samples are constructed by sliding context-future windows over longer sequences as illustrated in Figure~\ref{fig:intro-motivation} and neighboring windows can share part of their observations.
Many windows are simple repetitions of seasonal or slowly varying patterns and add little influence once those regularities have been learned as shown in Figure~\ref{fig:intro-motivation}(a).
At the opposite extreme, a context can be almost constant while its future contains an abrupt, unannounced change. Such a window produces a large loss but supplies no stable relation for a forecaster to learn as in Figure~\ref{fig:intro-motivation}(b). 
Between these extremes are nontrivial yet learnable windows that contain trends, changing seasonality, and moderate stochastic variation, illustrated by Figure~\ref{fig:intro-motivation}(c). 
Therefore, the central question of this paper is how to identify useful training subset from a multi-source corpus that can match or exceed the performance of training on the full dataset.

Loss-based data pruning methods have already been adopted in computer vision and language models.
In computer vision, importance-sampling methods prioritize samples using loss-derived bounds on per-sample gradient magnitude, while EL2N ranks images by their prediction-error norm as a tractable proxy for gradient-based importance \citep{katharopoulos2018notall,paul2021diet}. 
In LLM pretraining, a frozen reference model assigns each document a token-level cross-entropy or perplexity score and these scores are then used to rank and filter the corpus \citep{marion2023less}. 
A forecasting time series window has the same basic structure as a language sentence.
Nonetheless, text perplexity cannot be directly transferred to time series, which have continuous and scale-dependent targets, may contain missing future observations, and are optimized by TSFMs toward heterogeneous objectives such as squared error or quantile loss.
Existing time-series selection methods focus mainly on task-specific training or foundation-model fine-tuning \citep{taga2025adarho,wu2026tsrating}
and are not applicable to large-scale data due to high complexity, leaving data selection for TSFM pretraining largely unexplored.

\begin{figure}[t]
    \centering
    \includegraphics[width=\linewidth]{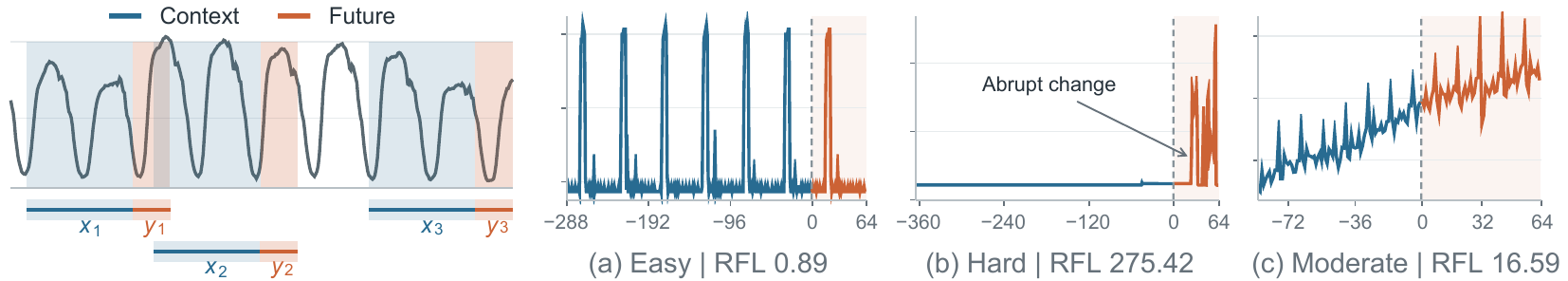}
    \caption{Sliding context-future windows over a long sequence and three real candidate windows. RFL is the reference loss of Chronos-Bolt on each original candidate window. }
    \label{fig:intro-motivation}
\end{figure}

We propose a loss-guided data selection framework for time series foundation model pretraining.
A frozen reference forecaster assigns each candidate context--future window a reference loss score as in Figure~\ref{fig:method-overview} (2).
The sample loss controls gradient norm under a local Jacobian condition, so loss is a principled proxy for how strongly a window can affect an update. 
Then dataset-stratified selection strategy retains an intermediate RFL interval within every source as in Figure~\ref{fig:method-overview} (3).
Specifically,
we convert RFL values into percentile ranks within each source dataset. For a retention ratio $\alpha$, keep the central $\alpha$ fraction of every source. 
It requires one offline reference-model pass, stores only one scalar per window, and does not compute gradients or modify target training online.
In conclusion,
RFL identifies a useful difficulty region, whereas stratification preserves source coverage.
We evaluate our method by training architecturally different TSFMs from scratch and 
achieving state-of-the-art results in the experiments. 

Our contributions are as follows:
\begin{itemize}
    \item To the best of our knowledge, we present the first systematic study of loss-based data selection at the window level for pretraining time series foundation models from scratch. This setting selects context--future windows once before target model optimization. 
    \item We formalize reference loss (RFL) as a difficulty score computed by a frozen forecaster and clarify that normalized MSE controls the gradient magnitude under an explicit local-Jacobian condition. We also introduce dataset-stratified selection, which retains an intermediate RFL interval independently within each source dataset and therefore preserves source coverage and dataset-level diversity.
    \item We train architecturally distinct TSFM families from scratch and 
    our method outperforms random selection by an absolute margin, and even surpasses training on the full sample set.
    Further studies of reference models and selection strategies characterize when reference scores transfer and validate the components of the framework.
\end{itemize}

\section{Related Work}
\subsection{Time-series foundation models}

Time-series foundation models (TSFMs) replace the conventional one-model-per-dataset workflow with a single forecaster pretrained on heterogeneous time series and transferred to unseen tasks. Representative models differ substantially in how they represent continuous observations and parameterize forecasts. 
TimesFM applies a decoder-only Transformer to patched continuous inputs \citep{das2024timesfm}.
Chronos scales and quantizes observations into a discrete vocabulary and trains a language-model backbone with cross-entropy \citep{ansari2024chronos} and Moirai combines masked-encoder pretraining with an any-variate architecture and flexible probabilistic outputs \citep{woo2024moirai}. 
Timer-XL \citep{liu2025timerxl} formulates univariate, multivariate, and covariate-informed forecasting as multivariate next-token prediction over long contexts. 
Time-MoE scales decoder-only forecasting through sparse mixture-of-experts layers and large multi-domain corpora \citep{shi2025timemoe}.
Sundial uses a flow-matching objective to generate flexible continuous predictive distributions without a prespecified parametric family \citep{liu2025sundial}. 
The primary focus of these models is architecture and data scaling and we focus on determining which context--future windows within a heterogeneous corpus should be used for pretraining.

\subsection{Data selection for large language models}

LLM pretraining has motivated data selection at both the document and domain levels. At the document level, large-scale pipelines combine heuristic or learned quality filters with exact and semantic deduplication to remove low-quality and redundant text \citep{albalak2024survey,lee2022deduplicating,abbas2023semdedup}. Model-based methods instead rank documents by reference-model cross-entropy or perplexity \citep{marion2023less}. At the domain level, DoReMi uses proxy-model excess loss to optimize mixture weights \citep{xie2023doremi}, DOGE estimates how training on one domain affects generalization to others \citep{fan2023doge}, and data-mixing laws extrapolate the performance of candidate mixtures from smaller training runs \citep{ye2025mixinglaws}. 
Online approaches further adapt sampling probabilities during training \citep{albalak2023odm,jiang2025ado}.  
Our study focuses on TSFM pretraining and we score context--future windows using reference-model loss.

\subsection{Data selection for time series}

Existing time series selection methods mostly address data valuation, task-specific model training, or foundation-model finetuning rather than pretraining a forecasting TSFM from scratch. 
TimeInf adapts influence functions to temporally dependent blocks and attributes model predictions to individual time points \citep{zhang2025timeinf}. 
TSRating distills pairwise LLM judgments into a cross-domain quality rater and evaluates selected subsets with conventional models and TSFM fine-tuning \citep{wu2026tsrating}.
AdaRho uses the reducible-loss difference between a target and an adaptively updated reference model for online filtering and augmentation on individual forecasting datasets \citep{taga2025adarho}. 
These methods demonstrate the value of time-series-aware selection, but operate in task-specific training or fine-tuning regimes and require repeated adaptation or influence computations.
Consequently, 
they are difficult to apply at TSFM pretraining scale.
To the best of our knowledge, this is the first systematic study of loss-based sample-level selection for pretraining of time series foundation models.

\section{Method}
\label{sec:method}

\subsection{Problem formulation}

Let the pretraining corpus be partitioned into $K$ source datasets,
$\mathcal{D}=\bigcup_{k=1}^{K}\mathcal{D}_k$. 
A sampled forecasting sample is
$z_i=(\vx_i,\vy_i,\vm_i)$, where $\vx_i\in\mathbb{R}^{C}$ is a context of length $C$,
$\vy_i\in\mathbb{R}^{H}$ is the future window and $H$ is the forecast horizon, $\vm_i\in\{0,1\}^{H}$ marks observed
future positions, where zero indicates a missing value. 
Our goal is to choose $\mathcal{S}\subset\mathcal{D}$ with $|\mathcal{S}|\approx
\alpha|\mathcal{D}|$ such that training a target forecaster on $\mathcal{S}$
retains or improves zero-shot performance while reducing the number of distinct
windows that must be stored and processed.

Figure~\ref{fig:method-overview} illustrates the complete framework. First, the reference loss score assigns every candidate window a difficulty value using a frozen forecaster. Second, dataset-stratified selection ranks these scores within each source dataset and constructs the retained pretraining subset while preserving source-level diversity.
\begin{figure}[t]
\centering
\includegraphics[width=\linewidth]{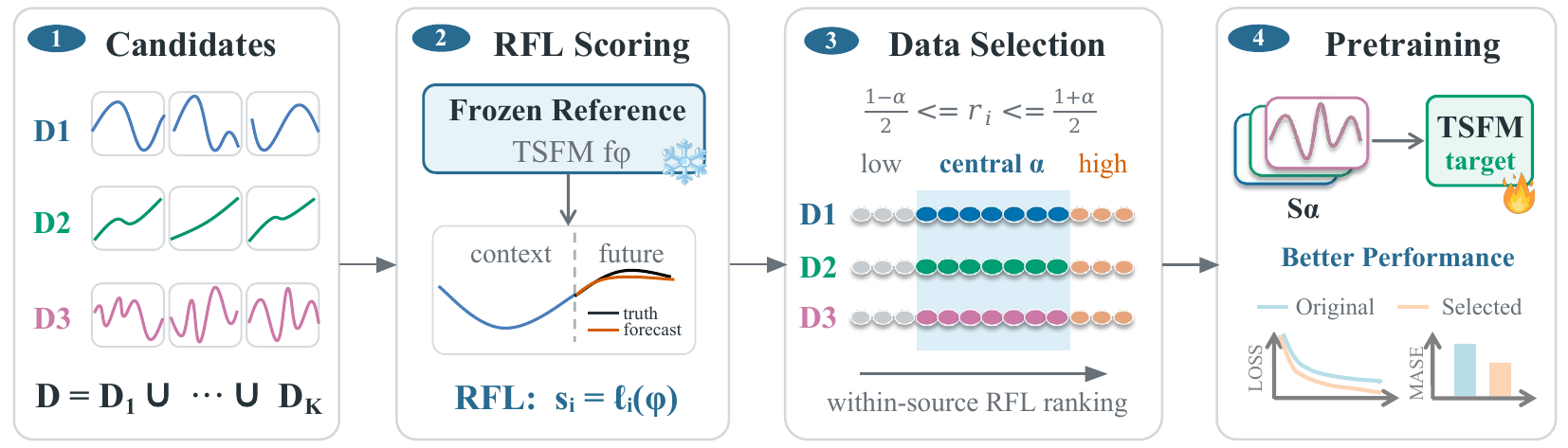}
\caption{Overview of the selection framework. (1) Candidate context--future windows are collected by their source datasets. (2) A frozen reference TSFM assigns each window an RFL score. (3) Scores are ranked within each source, and the central $\alpha$ interval is retained. (4) The selected subset is used to pretrain the target TSFM, 
aiming to improve performance relative to full-data pretraining. }
\label{fig:method-overview}
\end{figure}

\subsection{Reference loss}
\label{sec:rfl}

Time-series forecasting models are commonly trained with point or probabilistic objectives. Let $M_i=\sum_{h=1}^{H}m_{i,h}$ be the number of observed positions in the future window.
For a point forecaster $f_{\vtheta}$, the masked mean-squared error (MSE) is
\begin{equation}
  \mathcal L_i^{\mathrm{MSE}}(\vtheta)
  =\frac{1}{M_i}\sum_{h=1}^{H}m_{i,h}
  \left(f_{\vtheta}(\vx_i)_h-y_{i,h}\right)^2.
  \label{eq:mse-objective}
\end{equation}

MSE measures point accuracy but does not evaluate the full predictive distribution.
For a probabilistic forecaster with predictive cumulative distribution $F_{\vtheta,i,h}$, the continuous ranked probability score (CRPS) evaluates both calibration and sharpness.
Weighted Quantile Loss (WQL) is a common approximation of CRPS and is commonly adopted in probabilistic forecasting \citep{gneiting2007strictly}.
For simplicity, we mainly use the MSE metric for the analysis.

Reference loss connects forecasting objectives to data selection. Let $f_{\bm{\phi}}$ be a frozen reference forecaster and let $\ell_i(\vtheta)$ denote the native masked loss on window $z_i$. We define the reference loss score (RFL) as
\begin{equation}
  s_i^{\mathrm{RFL}}=\ell_i(\bm{\phi}).
  \label{eq:rfl-general}
\end{equation}
RFL records the per-window value of an existing forecasting objective under fixed reference model parameters
and a larger value indicates that the reference model explains the window less well.
Motivated by standard first-order analyses of training-data influence \citep{pruthi2020tracin,xia2024less}, we now formalize what RFL reveals about optimization difficulty. 

Consider the empirical risk of a model $f_{\vtheta}$ over the complete candidate pool $\mathcal D=\{z_j\}_{j=1}^{N}$,  $N=|\mathcal D|$ is the total number of candidate pretraining windows before selection. We have:
\begin{equation}
  R(\vtheta)=\frac{1}{N}\sum_{j=1}^{N}\ell_j(\vtheta),
  \qquad \vg_j=\nabla_{\vtheta}\ell_j(\vtheta),
  \qquad \bar{\vg}=\nabla_{\vtheta}R(\vtheta).
  \label{eq:risk}
\end{equation}

To examine the utility of a single window, consider one SGD update
$\vtheta^{+}=\vtheta-\eta\vg_i$. A first-order Taylor expansion of the
candidate-pool risk gives $R(\vtheta^{+})-R(\vtheta) =-\eta \langle \bar{\vg},\vg_i\rangle +o(\eta)$.
Thus, positive alignment with the candidate-pool gradient makes the leading-order risk change negative, whereas zero or negative alignment provides no first-order decrease. 
This separates two properties that a useful window should possess. Its gradient must be large enough to affect the parameters, but it must also align with the population gradient. 
A corrupted or intrinsically unpredictable future may have a large gradient yet weak or negative utility.

For the MSE objective in Equation~\ref{eq:mse-objective}, 
let $\vr_i=\widetilde{\vy}_i-f_{\vtheta}(\widetilde{\vx}_i)\in\mathbb{R}^{M_i}$ collect residuals only at observed future positions. 
For algebraic convenience, define the half-MSE, which differs from Equation~\ref{eq:mse-objective} by a positive constant and therefore leaves RFL rankings unchanged:
\begin{equation}
  \ell_i^{\mathrm{MSE}}(\vtheta)
  =\frac{1}{2}\mathcal L_i^{\mathrm{MSE}}(\vtheta)
  =\frac{1}{2M_i}\|\vr_i\|_2^2,
  \label{eq:mse-gradient}
\end{equation}
Then the connection between MSE loss and gradient magnitude is explicit. 

\begin{theorem}[MSE controls potential update magnitude]
\label{thm:mse-gradient}
For MSE in Equation~\ref{eq:mse-gradient}, the per-window gradient satisfies
\begin{equation}
  \vg_i=-\frac{1}{M_i}\mJ_i^{\top}\vr_i,
  \qquad
  \|\vg_i\|_2^2
  =\frac{1}{M_i^2}\vr_i^{\top}\mJ_i\mJ_i^{\top}\vr_i.
  \label{eq:grad-exact}
\end{equation}
where $\mJ_i=\partial f_{\vtheta}(\widetilde{\vx}_i)/\partial\vtheta$ is the Jacobian of the observed predictions.
Let $0\leq\lambda_i^{-}\leq\lambda_i^{+}$ be the smallest and largest eigenvalues of the positive-semidefinite matrix $\mJ_i\mJ_i^{\top}$. Then
\begin{equation}
  \frac{2\lambda_i^{-}}{M_i}\ell_i^{\mathrm{MSE}}
  \leq \|\vg_i\|_2^2
  \leq \frac{2\lambda_i^{+}}{M_i}\ell_i^{\mathrm{MSE}}.
  \label{eq:loss-grad-bound}
\end{equation}
\end{theorem}

The proof is provided in Appendix~\ref{app:proof-mse-gradient}.
Theorem~\ref{thm:mse-gradient} states the precise sense in which normalized MSE is an optimization-difficulty score. If $M_i$ and local Jacobian conditioning are comparable across windows, ranking their losses approximately ranks the magnitudes of the updates they can induce. For a nonnegative $\beta$-smooth objective, the standard self-bounding inequality $\|\nabla\ell_i\|_2^2\leq2\beta\ell_i$ provides the one-sided result that very low loss rules out a large gradient \citep{srebro2010smoothness}. Evaluating these relations at the frozen reference parameters $\bm\phi$ motivates using $s_i^{\mathrm{RFL}}$ as a cheap difficulty coordinate.

\subsection{Dataset-stratified selection}
\label{sec:stratified}

Dataset-stratified selection applies RFL-based difficulty filtering while preserving the diversity of the pretraining corpus.
We define the candidate corpus as $\mathcal D=\{z_i\}_{i=1}^{N}=\bigcup_{k=1}^{K}\mathcal{D}_k$, where $K$ is the number of source datasets, $\mathcal D_k$ is the set of candidate windows from source $k$, and $n_k=|\mathcal D_k|$. Moreover, each window $z_i$ has an RFL score $s_i^{\mathrm{RFL}}$ and a source label $g_i\in\{1,\ldots,K\}$, such that $z_i\in\mathcal D_{g_i}$.

The preceding first-order Taylor expansion analysis motivates retaining an intermediate RFL interval.
Low-RFL windows are expected to provide limited update magnitude, whereas the high-RFL windows are more likely to contain a mixture of rare useful patterns, distribution shifts, and unlearnable noise. 
By contrast, intermediate RFL windows retain nontrivial residuals while remaining close enough to learned temporal structure that their gradients may align with recurring patterns.
This interpretation is consistent with empirical findings on moderate-score and intermediate-perplexity data selection in LLM and CV \citep{xia2023moderate,marion2023less}.

Therefore, we first consider \emph{global selection}. It implements this conclusion by pooling all windows, ranking them by absolute RFL, and retaining the central $\alpha$ fraction. Its corpus-level empirical cumulative distribution and selected subset are
\begin{equation}
  \widehat F_{\mathrm{all}}(t)
  =\frac{1}{N}\sum_{j=1}^{N}
  \mathbb{I}[s_j^{\mathrm{RFL}}\leq t],
  \qquad
  \mathcal{S}_{\alpha}^{\mathrm{global}}
  =\left\{z_i\in\mathcal D:\frac{1-\alpha}{2}
  \leq \widehat F_{\mathrm{all}}(s_i^{\mathrm{RFL}})
  \leq\frac{1+\alpha}{2}\right\}.
  \label{eq:global-selection}
\end{equation}
Here, $t$ is an RFL threshold and $\mathbb{I}[\cdot]$ is the indicator function. The desired retention ratio $\alpha\in(0,1]$ specifies the fraction of the complete corpus to retain, and $\mathcal S_{\alpha}^{\mathrm{global}}$ contains the windows between the $(1-\alpha)/2$ and $(1+\alpha)/2$ quantiles of the pooled distribution. 

Global selection controls $|\mathcal S_{\alpha}^{\mathrm{global}}|\approx\alpha N$, but it does not control the source-specific retention rate. 
Absolute RFL distributions can differ across datasets because of their noise levels, temporal patterns, and intrinsic predictability.
Consequently, a source whose scores concentrate near either global tail can be severely underrepresented or removed,
which may harm dataset diversity.

Dataset-stratified selection instead measures the difficulty of each window relative to other windows from the same source. For source $k$, we define the source-specific empirical cumulative distribution $\widehat F_k$ and the within-source percentile rank $r_i$ as
\begin{equation}
  \widehat F_k(t)
  =\frac{1}{n_k}\sum_{j:g_j=k}
  \mathbb{I}[s_j^{\mathrm{RFL}}\leq t],
  \qquad
  r_i=\widehat F_{g_i}(s_i^{\mathrm{RFL}}).
  \label{eq:percentile}
\end{equation}

Here, $r_i\in(0,1]$ is the fraction of windows in $\mathcal D_{g_i}$ whose RFL score does not exceed that of window $i$. We then retain the central $\alpha$ fraction independently within every source:
\begin{equation}
  \mathcal{S}_{\alpha}\equiv\mathcal{S}_{\alpha}^{\mathrm{strat}}
  =\left\{z_i\in\mathcal D:\frac{1-\alpha}{2}\leq r_i
  \leq\frac{1+\alpha}{2}\right\}.
  \label{eq:stratified-selection}
\end{equation}

Equation~\ref{eq:stratified-selection} retains $\alpha n_k$ windows from every source $k$. The proportion of source $k$ in the selected subset is therefore
\begin{equation}
  p_k^{\mathcal S}
  =\frac{|\mathcal S_{\alpha}\cap\mathcal D_k|}{|\mathcal S_{\alpha}|}
  =\frac{\alpha n_k}{\alpha N}
  =\frac{n_k}{N}=p_k,
  \label{eq:mixture-preservation}
\end{equation}
where $p_k=n_k/N$ is the proportion of source $k$ in the original candidate corpus. Dataset-stratified selection thus filters windows by relative difficulty while preserving source coverage and dataset-level diversity. 
We analyze the benefit of the framework in Section~\ref{sec: effect4strat}.

\subsection{Algorithm and computational cost}

Algorithm~\ref{alg:rfl-selection} summarizes the two components of the framework. We first assign every candidate window an RFL score with a frozen reference model. We then rank the scores separately within each source dataset and retain the central interval. 
The reference and target architectures may differ, which permits an inexpensive proxy to amortize selection over several target models.

The offline algorithm requires one reference-model forward pass per candidate window.
It additionally requires $\sum_{k=1}^{K}O(n_k\log n_k)$ time to sort the scores within sources and stores one scalar per window. The generated scores can be reused across target-model sizes and architectures, amortizing the scoring cost.

\begin{algorithm}[t]
\caption{Reference-loss scoring and dataset-stratified selection}
\label{alg:rfl-selection}
\begin{algorithmic}[1]
\REQUIRE Candidate corpus $\mathcal D=\bigcup_{k=1}^{K}\mathcal D_k$; frozen reference model $f_{\bm\phi}$; retention ratio $\alpha$

\ENSURE Selected subset $\mathcal S_{\alpha}$ and trained target parameters $\vtheta$
\STATE $\mathcal S_{\alpha}\leftarrow\varnothing$
\FOR{each candidate window $z_i\in\mathcal D$}
  \STATE Compute its RFL score $s_i^{\mathrm{RFL}}\leftarrow\ell_i(\bm\phi)$
\ENDFOR
\FOR{each source dataset $k\in\{1,\ldots,K\}$}
  \STATE Sort the $n_k$ windows by increasing RFL score
  \STATE Denote the resulting order by $z_{\pi_k(1)},\ldots,z_{\pi_k(n_k)}$
  \STATE $m_k\leftarrow\max\{1,\operatorname{round}(\alpha n_k)\}$
  \STATE $a_k\leftarrow\lfloor(n_k-m_k)/2\rfloor+1$
  \STATE $\mathcal S_{\alpha}\leftarrow\mathcal S_{\alpha}\cup\{z_{\pi_k(j)}:a_k\leq j<a_k+m_k\}$
\ENDFOR
\STATE Train the randomly initialized target model $f_{\vtheta}$ on $\mathcal S_{\alpha}$
\RETURN $\mathcal S_{\alpha},\vtheta$
\end{algorithmic}
\end{algorithm}

\section{Experiments}
\label{sec:experiments}

We organize the experiments around three questions.
(Q1) can dataset-stratified RFL selection improve full-data pretraining while retaining fewer training samples? 
(Q2) can RFL scores be transferred across reference-model scales and architectures?
(Q3) how important are the intermediate interval and dataset stratification?

\subsection{Experimental setup}
\label{exp:setup}

\paragraph{Data.}

The pretraining data mainly comes from the Chronos data collection \citep{ansari2024chronos} and LOTSA \citep{woo2024moirai}.
It contains more than 100 source datasets and 4.64M univariate series,
spanning climate and weather, transportation, health and mobility data, industrial measurements, synthetic series generated via
Gaussian processes\citep{ansari2024chronos} and so on. 
We build validation windows from the last complete horizon of each series and keep them separate from the training windows.  
All architectural variants use the same validation samples to ensure comparable validation losses.

\paragraph{Training.}
We pretrain three representative TSFM families: Chronos-Bolt Base \citep{ansari2024chronos} (encoder-decoder), Moirai-Base \citep{woo2024moirai} (encoder), and TimesFM 2.5 \citep{das2024timesfm} (decoder). 
All models use AdamW, cosine learning-rate decay, gradient clipping, and an effective batch size of 512 on NVIDIA A100 GPUs. The initial learning rate is $3\times10^{-4}$ for TimesFM and $1\times10^{-3}$ for the other models, with 5,000 and 2,000 warm-up steps, respectively. We set the training budget to 40,000 optimizer steps for TimesFM and 100,000 steps for others. 
The max context length is set to 2,048 and prediction length is architecture-specific: 64 for Chronos-Bolt, 96 for Moirai, and 128 for TimesFM.

\paragraph{Evaluation.}
We use GIFT-Eval \citep{aksu2024gifteval} and remove training-overlapping benchmarks. The current evaluation contains 91 task configurations.   
Following GIFT-Eval, we normalize MASE and CRPS by a seasonal-naive forecaster and report the geometric mean across tasks. 
We also report MSE and MAE on five commonly used benchmarks in ablation study:  ETTh1, ETTh2, ETTm1, ETTm2, and weather.
These datasets have been utilized for benchmarking and publicly available on \citep{wu2021autoformer}.
None of these datasets is included in pretraining datasets.

\subsection{Overall Comparison}
\label{sec:main-results}

We compare dataset-stratified RFL selection with uniform random selection and full data pretraining under the same computational budget.
To test whether the selection rule generalizes beyond model design, we pretrain three distinct target architectures and score their candidate pools using frozen references from the corresponding architecture. 
Table~\ref{tab:main} shows that our method outperforms random selection by an absolute margin, indicating that RFL provides a useful selection signal across architectures.  
Since the random control retains the same number of windows, this consistency attributes the gain to which windows are retained.
At 80\% retention, our method improves over both full-data pretraining and a size-matched random subset on all target models. 
This is because we remove high-loss noise and harmful data and prune unnecessary low-loss data while maintaining diversity.

\begin{table*}[t]
\caption{GIFT-Eval comparison at different retention ratios. Lower MASE and CRPS are better. Full pretraining uses 100\% of the candidate windows. 
Bold denotes the best selection method.}
\label{tab:main}
\centering
\small
\setlength{\tabcolsep}{6pt}
\begin{tabular}{llc@{\hspace{12pt}}cc@{\hspace{12pt}}cc@{\hspace{12pt}}cc}
\toprule
\multirow{2}{*}{Target model} & \multirow{2}{*}{Metric $\downarrow$} & \multirow{2}{*}{Full (100\%)}
& \multicolumn{2}{c}{40\% retention}
& \multicolumn{2}{c}{60\% retention}
& \multicolumn{2}{c}{80\% retention} \\
\cmidrule(lr){4-5}\cmidrule(lr){6-7}\cmidrule(l){8-9}
& & & Random & Ours & Random & Ours & Random & Ours \\
\midrule
\multirow{2}{*}{Chronos-Bolt Base}
& MASE & 0.796 & 0.836 & 0.803 & 0.806 & 0.801 & 0.799 & \textbf{0.777} \\
& CRPS & 0.555 & 0.576 & 0.569 & 0.555 & 0.554 & 0.549 & \textbf{0.542} \\
\midrule
\multirow{2}{*}{Moirai-Base}
& MASE & 1.017 & 1.026 & 1.050 & 1.061 & 1.031 & 1.022 & \textbf{1.005} \\
& CRPS & 0.698 & 0.708 & 0.749 & 0.732 & 0.731 & 0.701 & \textbf{0.680} \\
\midrule
\multirow{2}{*}{TimesFM 2.5}
& MASE & 0.813 & 0.833 & 0.821 & 0.840 & 0.823 & 0.835 & \textbf{0.804} \\
& CRPS & 0.558 & 0.573 & 0.561 & 0.574 & 0.560 & 0.576 & \textbf{0.554} \\
\bottomrule
\end{tabular}
\end{table*}

\subsection{Reference-model analysis}
\label{sec:reference}

We investigate how the choice of reference model affects data selection along two dimensions: scale transfer within an architecture and transfer across architectures. 
We use Chronos-Bolt Base and Moirai-Base as target models and retain 80\% of the candidate windows. For within-architecture transfer, we vary the frozen reference among Chronos-Bolt Mini and Small, or between Moirai-Small and Base. For cross-architecture transfer, a frozen TimesFM reference selects data for each target. This design tests whether RFL requires a capacity-matched reference or can instead reuse a smaller or architecturally different pretrained forecaster.

Within-architecture references are largely robust to model scale. As shown in Table~\ref{tab:reference}, all subset configurations improve MASE and CRPS over full-data pretraining, indicating small reference models produce competitive training subsets.
The score analysis in Appendix~\ref{app:reference-analysis} explains this robustness. 
In brief, the pairwise Spearman correlations among Chronos-Bolt families range from $0.977$ to $0.986$, and Moirai families obtain $\rho=0.991$. 
This agreement indicates that models from the same architecture share a stable ordering of window difficulty. Consequently, a smaller pretrained reference can often approximate the ranking of a larger model and reduce the cost of the required forward scoring pass.

Cross-architecture transfer is conditional. A TimesFM reference selects a useful subset for Chronos-Bolt Base target, improving full-data pretraining and remaining competitive with the best reference. 
In contrast, the TimesFM-selected subset
performs poorly in Moirai-Base pretraining, degrading to the level of random selection in Table~\ref{tab:main}.
Indeed, TimesFM has strong RFL score rank agreement with Chronos-Bolt Mini on the candidate pool ($\rho=0.828$), but only weak agreement with Moirai-Small ($\rho=0.254$).
For more analysis details, please see Appendix~\ref{app:reference-analysis}.
A plausible explanation is that TimesFM scores point and quantile errors, whereas Moirai uses a distributional negative log-likelihood, so they may not assign similar difficulty to high-uncertainty windows.

\begin{table}[t]
\caption{Evaluation results of cross-scale and cross-architecture reference models with 80\% retention. Bold and underlined values are best and second best for each target model.
}
\label{tab:reference}
\centering
\small
\setlength{\tabcolsep}{6pt}
\begin{tabular}{llccc}
\toprule
Target Model & Reference Model & Retained & MASE & CRPS \\
\midrule
\multirow{4}{*}{Chronos-Bolt Base}
& Chronos-Bolt Mini & 80\% & \textbf{0.777} & \underline{0.542} \\
& Chronos-Bolt Small & 80\% & \underline{0.786} & 0.545 \\

& TimesFM & 80\% & \underline{0.786} & \textbf{0.541} \\
& --- (Full data) & 100\% & 0.796 & 0.555 \\
\midrule
\multirow{3}{*}{Moirai-Base}
& Moirai-Small & 80\% & \textbf{1.003} & \underline{0.695} \\
& Moirai-Base & 80\% & \underline{1.005} & \textbf{0.680} \\
& TimesFM & 80\% & 1.022 & 0.708 \\
& --- (Full data) & 100\% & 1.017 & 0.698 \\
\bottomrule
\end{tabular}
\end{table}

\subsection{Ablation studies}
\label{sec:ablations}

\subsubsection{RFL region}
\label{sec:regions}

Dataset-stratified selection retains the central RFL interval by default. To isolate the effect of this choice, we compare three score-based subsets constructed independently within every source.
Bottom retains the lowest-RFL tail, Top retains the highest-RFL tail, and Central retains the central interval. 
We additionally include a dataset-stratified random baseline that samples windows uniformly at random within each source, retaining the prescribed fraction from every source.
All runs in this ablation use TimesFM as both the reference and target model. We evaluate 40\% and 80\% retention on five forecasting benchmarks in Table~\ref{tab:rfl-regions} and compare matched validation-loss trajectories through 40k optimization steps against source-matched random selection and full-data pretraining in Figure~\ref{fig:region-val-curves}.

Table~\ref{tab:rfl-regions} demonstrates the advantages of the Central selection,
consistent with the optimization interpretation in Section~\ref{sec:rfl}.
Bottom retains low-RFL windows that the reference model already predicts well and are therefore relatively easy. Concentrating on this region can overallocate training capacity to redundant patterns with little residual learning signal. One possible explanation for this redundancy is that intermediate-difficulty windows contain the regularities needed for the easier cases. Once the target model learns these richer patterns, it can readily generalize to low-difficulty windows without repeatedly training on them. 
By contrast, Top retains high-difficulty windows. Although some of these windows may contain rare informative structure, a high loss can also arise from noise, abrupt unpredictable transitions, or distribution shifts, making the entire upper tail less reliably learnable. 
Central  avoids both extremes and retains windows with nontrivial prediction error but sufficient regularity to provide stable learning signal.

The validation trajectories in Figure~\ref{fig:region-val-curves}(a,b) reinforce this interpretation. Central remains among the lowest-loss trajectories over much of training and separates most clearly when the retained budget is small. 
This suggests that the benefit comes from excluding both simple windows and the most unstable tail, not from preferring uniformly easier or harder data.

\begin{table*}[t]
\caption{RFL-region ablation at 40\% and 80\% retention. During downstream evaluation, each dataset uses a context window of 512 and a prediction window of 96. 
}
\label{tab:rfl-regions}
\centering
\small
\setlength{\tabcolsep}{6pt}
\begin{tabular}{@{}cl*{5}{cc}@{}}
\toprule
\multirow{2}{*}{Ratio} & \multirow{2}{*}{Region}
& \multicolumn{2}{c}{ETTh1} & \multicolumn{2}{c}{ETTh2}
& \multicolumn{2}{c}{ETTm1} & \multicolumn{2}{c}{ETTm2}
& \multicolumn{2}{c}{Weather} \\
\cmidrule(lr){3-4}\cmidrule(lr){5-6}\cmidrule(lr){7-8}
\cmidrule(lr){9-10}\cmidrule(lr){11-12}
& & MSE & MAE & MSE & MAE & MSE & MAE & MSE & MAE & MSE & MAE \\
\midrule
\multirow{4}{*}{40\%}
& Bottom
& 0.404 & 0.393 & 0.307 & 0.342
& 0.351 & 0.359 & 0.192 & 0.260
& 0.182 & 0.214 \\
& Top
& 0.394 & 0.411 & 0.315 & 0.358
& 0.357 & 0.380 & 0.197 & 0.273
& 0.180 & 0.221 \\
& Random
& 0.388 & 0.398 & \textbf{0.302} & 0.348 
& 0.347 & 0.365 & \textbf{0.182} & 0.259
& 0.176 & 0.212 \\
& Ours
& \textbf{0.383} & \textbf{0.390} & \textbf{0.302} & \textbf{0.337}
& \textbf{0.341} & \textbf{0.357} & 0.186 & \textbf{0.258}
& \textbf{0.169} & \textbf{0.202} \\
\midrule
\multirow{4}{*}{80\%}
& Bottom
& 0.386 & 0.391 & 0.300 & 0.337
& 0.348 & \textbf{0.357} & 0.184 & \textbf{0.255}
& 0.172 & 0.206 \\
& Top
& 0.391 & 0.400 & 0.305 & 0.349
& 0.346 & 0.366 & 0.182 & 0.259
& 0.171 & 0.209 \\
& Random
& 0.379 & 0.390 & 0.294 & 0.347 
& 0.345 & 0.367 & 0.189 & 0.262
& 0.173 & 0.210 \\
& Ours
& \textbf{0.373} & \textbf{0.388} & \textbf{0.287} & \textbf{0.334}
& \textbf{0.339} & 0.359 & \textbf{0.181} & \textbf{0.255}
& \textbf{0.170} & \textbf{0.204} \\
\midrule
100\% & Full & 0.397   & 0.398  & 0.297  &  0.342
&  0.350 & 0.363  &  0.187 &  0.261
&  0.175 & 0.210  \\
\bottomrule
\end{tabular}
\end{table*}

\begin{figure*}[t]
\centering
\includegraphics[width=\textwidth]{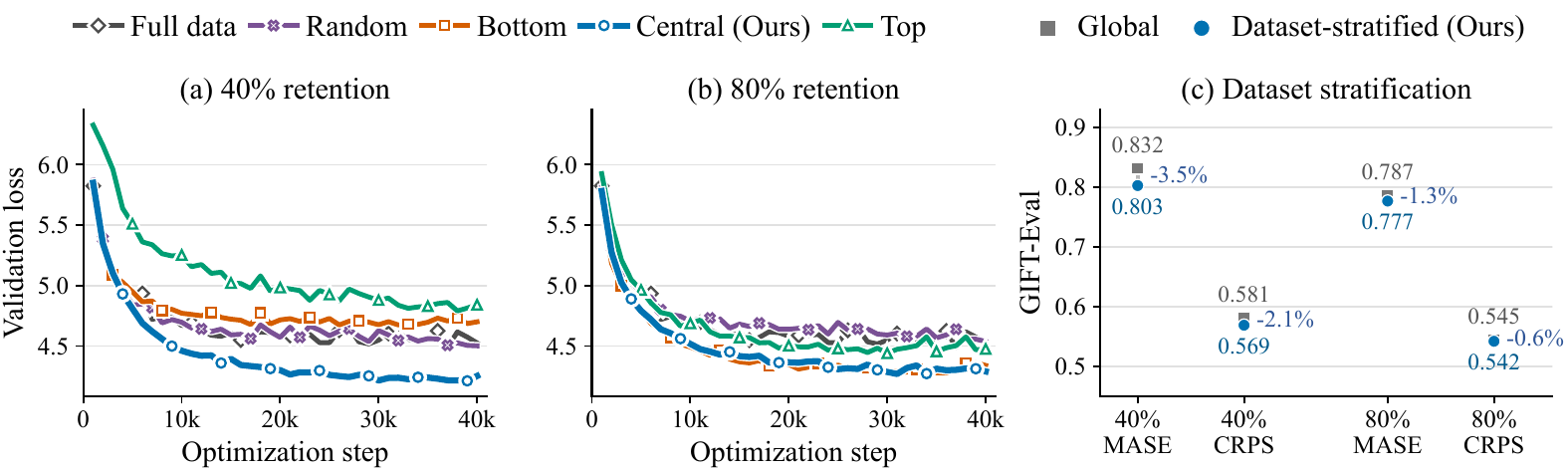}
\caption{Ablations of the two components of our framework. (a,b) Validation-loss trajectories for full-data pretraining, random selection and the Bottom, Central (Ours), and Top RFL regions at 40\% and 80\% retention. (c) GIFT-Eval results for dataset-stratified selection and global selection at  40\% and 80\% retention.
}
\label{fig:region-val-curves}
\end{figure*}

\subsubsection{Dataset stratification selection}
\label{sec: effect4strat}

Dataset stratification is the diversity-preserving component of our framework. We compare the proposed dataset-stratified rule in Equation~\ref{eq:stratified-selection} with global selection in Equation~\ref{eq:global-selection}. Both rules retain the same central RFL interval and use the same total retention ratio. They differ only in whether percentile ranks are computed within each source or over the pooled corpus. 
We use the default reference model to score the candidate pool and pretrain a Chronos-Bolt Base target model on the selected subsets. Figure~\ref{fig:region-val-curves}(c) shows that the stratified rule improves GIFT-Eval metrics at all retention ratios. 
This improvement indicates that global percentiles harm the data diversity,
whereas within-source percentiles preserve coverage. The persistence of the advantage from 40\% to 80\% supports dataset stratification as an effective diversity constraint.

\section{Conclusion}

This paper presents a reference-loss selection framework for selecting training windows for time-series foundation-model pretraining.  
We introduce reference loss as a model-compatible difficulty score computed by a frozen forecaster. 
Our analysis connects loss to potential update magnitude while clarifying that actual one-step utility also depends on gradient alignment. 
Based on this, our dataset-stratified selection retains an intermediate interval within every source dataset. 
The resulting subset removes low-loss windows with limited residual learning signal and avoids overconcentrating on the unstable high-loss tail, while preserving source coverage and dataset-level diversity.
Across three distinct TSFM families, our method outperforms size-matched random selection by an absolute margin, and even surpasses the training on the full sample set.
Together, these findings establish pretraining-data selection as an important design axis alongside model architecture and corpus scale for time series foundation models.

\clearpage

\bibliography{main}
\bibliographystyle{plainnat}

\appendix

\section{Proof of Theorem~\ref{thm:mse-gradient}}
\label{app:proof-mse-gradient}

\begin{proof}
Recall that $\vr_i=\widetilde{\vy}_i-f_{\vtheta}(\widetilde{\vx}_i)$ contains the $M_i$ residuals at observed future positions and that
\begin{equation}
  \ell_i^{\mathrm{MSE}}(\vtheta)
  =\frac{1}{2M_i}\vr_i^{\top}\vr_i.
  \label{eq:appendix-half-mse}
\end{equation}
Because $\partial\vr_i/\partial\vtheta=-\mJ_i$, the chain rule gives
\begin{equation}
  \vg_i
  =\nabla_{\vtheta}\ell_i^{\mathrm{MSE}}(\vtheta)
  =\frac{1}{M_i}
    \left(\frac{\partial\vr_i}{\partial\vtheta}\right)^{\!\top}\vr_i
  =-\frac{1}{M_i}\mJ_i^{\top}\vr_i.
  \label{eq:appendix-mse-gradient}
\end{equation}
Taking its squared Euclidean norm yields
\begin{equation}
  \|\vg_i\|_2^2
  =\frac{1}{M_i^2}
    (\mJ_i^{\top}\vr_i)^{\top}(\mJ_i^{\top}\vr_i)
  =\frac{1}{M_i^2}
    \vr_i^{\top}\mJ_i\mJ_i^{\top}\vr_i,
\end{equation}
which proves Equation~\ref{eq:grad-exact}.

Let $\mA_i=\mJ_i\mJ_i^{\top}$. This matrix is positive semidefinite because, for every vector $\mathbf u$,
$\mathbf u^{\top}\mA_i\mathbf u=\|\mJ_i^{\top}\mathbf u\|_2^2\geq0$.
Applying the Rayleigh--Ritz inequality to $\mA_i$ gives
\begin{equation}
  \lambda_i^{-}\|\vr_i\|_2^2
  \leq \vr_i^{\top}\mA_i\vr_i
  \leq \lambda_i^{+}\|\vr_i\|_2^2.
  \label{eq:appendix-rayleigh}
\end{equation}
Finally, Equation~\ref{eq:appendix-half-mse} implies
$\|\vr_i\|_2^2=2M_i\ell_i^{\mathrm{MSE}}$.
Substituting this identity into Equation~\ref{eq:appendix-rayleigh} and dividing by $M_i^2$ yields
\begin{equation}
  \frac{2\lambda_i^{-}}{M_i}\ell_i^{\mathrm{MSE}}
  \leq \|\vg_i\|_2^2
  \leq \frac{2\lambda_i^{+}}{M_i}\ell_i^{\mathrm{MSE}},
\end{equation}
which is Equation~\ref{eq:loss-grad-bound}.
\end{proof}

\section{Additional reference-model analysis}
\label{app:reference-analysis}

Within-architecture RFL rankings are largely insensitive to reference scale. Figure~\ref{fig:reference-correlation}(a) compares Chronos-Bolt Mini, Small, and Base on the same candidate pool. Their pairwise Spearman correlations remain high even though the models differ in capacity. 
Moreover, we also compared the RFL scores computed with Moirai-base and Moirai-small as reference models, respectively, and the Spearman correlation between them is 0.991.
This near-invariance indicates that the same architecture shares a stable notion of relative window difficulty. 
This agreement helps explain why the smaller reference models in Table~\ref{tab:reference} produce competitive training subsets.

Cross-architecture reuse depends on whether the reference and target families induce compatible difficulty orderings. On the Chronos-Bolt-Base candidate pool, TimesFM and Chronos-Bolt Mini retain strong window-level rank agreement ($\rho=0.828$ in Figure~\ref{fig:reference-correlation}(b)). Because dataset-stratified selection depends on ranks rather than absolute loss magnitudes, this monotonic agreement is sufficient for TimesFM to identify a useful Chronos subset, consistent with the successful Chronos transfer summarized in Section~\ref{sec:reference}.

The Moirai pool illustrates a limitation of this transfer. TimesFM and Moirai-Small obtain only weak rank agreement ($\rho=0.254$ in Figure~\ref{fig:reference-correlation}(c)), and the TimesFM-selected subset underperforms both the Moirai-Small subset and full-data pretraining. A plausible explanation is that TimesFM scores point and quantile errors, whereas Moirai uses a distributional negative log-likelihood, so the two references need not assign the same difficulty to high-uncertainty windows.

\begin{figure*}[t]
\centering
\includegraphics[width=\textwidth]{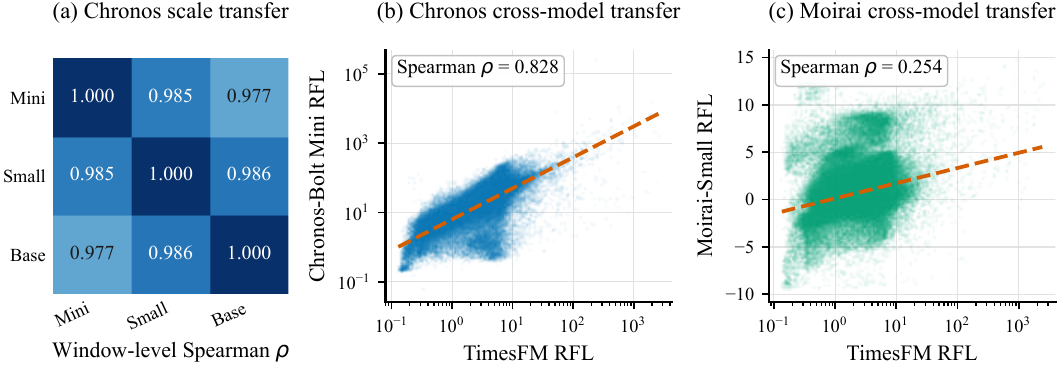}
\caption{Reference-model agreement on target-specific candidate pools. (a) Window-level Spearman correlations among Chronos-Bolt Mini, Small, and Base over all candidate windows for the Chronos-Bolt-Base target. (b) TimesFM versus Chronos-Bolt Mini on the same Chronos target pool. (c) TimesFM versus Moirai-Small on the Moirai-Base target pool. Spearman coefficients use global ranks over every aligned window. For legibility, each scatter panel displays a deterministic sample of 150k windows, while the reported correlation and dashed ordinary-least-squares fit use all windows.}
\label{fig:reference-correlation}
\end{figure*}

\begin{figure}[ht]
\centering
\includegraphics[width=\linewidth]{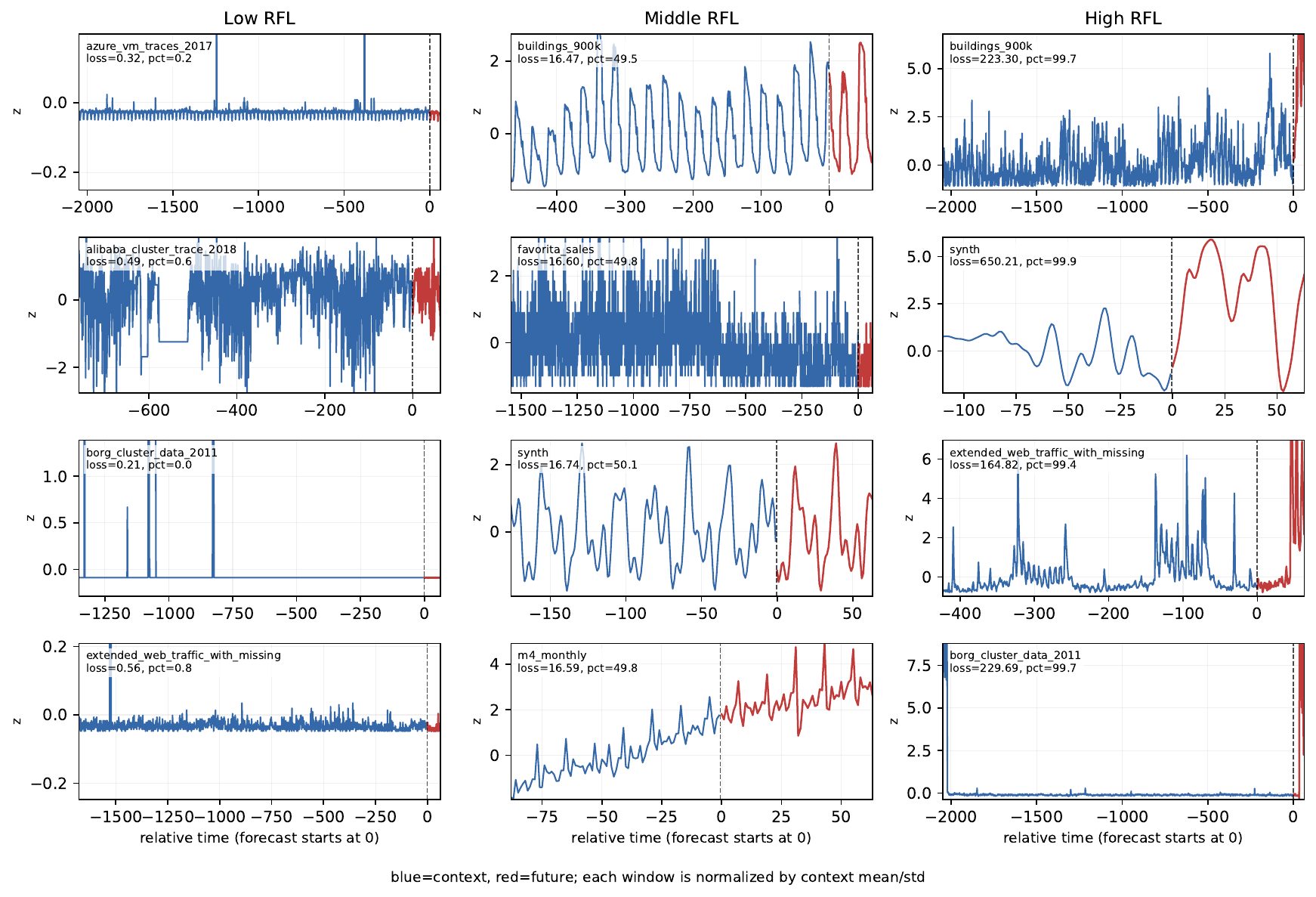}
\caption{Representative windows from low, middle, and high Chronos-Bolt Mini RFL regions. Blue denotes the context and red denotes the forecast horizon. 
Each window is standardized by its valid context mean and standard deviation. 
The \texttt{pct} value in each panel is the global empirical percentile rank of that window's RFL score among all candidate windows.
Low-RFL windows are often highly regular, middle-RFL windows retain learnable nontrivial structure, and high-RFL windows often contain abrupt shifts, extreme future values, or context--future mismatch.}
\label{fig:rfl-qualitative-examples}
\end{figure}

\section{Qualitative analysis}
\label{app:rfl-qualitative}

We qualitatively inspect the Chronos-Bolt Mini RFL scores used for dataset-stratified selection. The score file contains all candidate training windows from the original sample cache. For visualization only, we form three global score regions: low RFL (bottom 1\%), middle RFL (49.5--50.5\%), and high RFL (top 1\%). Selection in the main experiments remains dataset-stratified and the global bands here are used only to make the score semantics visually interpretable.

Figure~\ref{fig:rfl-qualitative-examples} supports the intended interpretation of RFL as a difficulty coordinate. Low-score samples tend to have small residual signal after context normalization. Middle-score samples include visible trend, seasonality, or amplitude changes but remain connected to the context. High-score samples frequently show future behavior that is hard to extrapolate from the context alone, such as sudden jumps or large tail events. This motivate removing both extremes when constructing a compact pretraining subset: low-score windows can be redundant, while the high-score tail mixes rare events with weakly learnable windows.

\end{document}